\documentclass{article}

\usepackage{microtype}
\usepackage{graphicx}
\usepackage{subfigure}
\usepackage{booktabs} 

\usepackage{hyperref}

\usepackage[utf8]{inputenc} 
\usepackage{amsfonts}       
\usepackage{amsmath}
\usepackage{amssymb}
\usepackage{bbm}
\usepackage{multirow}
\usepackage{mathtools}
\usepackage{relsize}

\def\b0{{0}}

\def\RR{\mathbb{R}}

\def\>{\rangle}

\def\vec{\operatorname{\mathop{vec}}}

\def\Set#1{\left\{ #1 \right\}}

\newcommand{\E}{\mathbb{E}}

\newcommand{\distas}[1]{\mathbin{\overset{#1}{\sim}}}

\newtheorem{theorem}{Theorem}[section]
\newtheorem{lemma}[theorem]{Lemma}

\newtheorem{corollary}[theorem]{Corollary}

\newtheorem{assumptions}[theorem]{Assumption}

\newenvironment{proof}{\par\noindent{\bf Proof:\ }}{\hfill$\Box$\\[2mm]}

\newcommand{\Id}{\mathbb{I}}
\newcommand{\littleO}[1]{o\left(#1\right)}
\newcommand{\bigO}[1]{\mathcal{O}\left(#1\right)}
\newcommand{\bigOmg}[1]{\Omega\left(#1\right)}
\newcommand{\bigTildeOmg}[1]{\tilde{\Omega}\left(#1\right)}
\newcommand{\bigTheta}[1]{\Theta\left(#1\right)}

\newcommand{\bigexp}[1]{\exp\left(#1\right)}
\newcommand{\norm}[1]{\left\|#1\right\|}

\newcommand{\abs}[1]{\left|#1\right|}

\newcommand{\svmin}[1]{\sigma_{\rm min}\left(#1\right)}

\newcommand{\evmin}[1]{\lambda_{\rm min}\left(#1\right)}

\def\Lip{\mathrm{Lip}}

\def\bydef{\mathrel{\mathop:}=}

\def\PP{\mathbb{P}}

\def\tr{\mathop{\rm tr}\nolimits}

\def\min{\mathop{\rm min}\nolimits}
\def\max{\mathop{\rm max}\nolimits}

\usepackage[accepted]{icml2021}

\icmltitlerunning{On the Proof of Global Convergence of Gradient Descent for Deep ReLU Networks with Linear Widths}

\begin{document}

\twocolumn[
\icmltitle{On the Proof of Global Convergence of Gradient Descent for \\Deep ReLU Networks with Linear Widths}



\icmlsetsymbol{equal}{*}

\begin{icmlauthorlist}
    \icmlauthor{Quynh Nguyen}{mpi}
\end{icmlauthorlist}

\icmlaffiliation{mpi}{MPI-MIS, Germany}

\icmlcorrespondingauthor{Quynh Nguyen}{qnguyen@mis.mpg.de}

\icmlkeywords{Theory of deep learning, gradient descent}

\vskip 0.3in
]



\printAffiliationsAndNotice{}  
\begin{abstract}
    We give a simple proof for the global convergence of gradient descent in training deep ReLU networks with the standard square loss,
    and show some of its improvements over the state-of-the-art.
    In particular, while prior works require
    all the hidden layers to be wide with width at least $\Omega(N^8)$ ($N$ being the number of training samples),
    we require a single wide layer of linear, quadratic or cubic width depending on the type of initialization.
    Unlike many recent proofs based on the Neural Tangent Kernel (NTK), our proof need {\em not} 
    track the evolution of the entire NTK matrix, 
    or more generally, any quantities related to the changes of activation patterns during training.
    Instead, we only need to track the evolution of the output at the last hidden layer, 
    which can be done much more easily thanks to the Lipschitz property of ReLU.
    Some highlights of our setting: (i) all the layers are trained with standard gradient descent, 
    (ii) the network has standard parameterization as opposed to the NTK one,
    and (iii) the network has a single wide layer as opposed to having all wide hidden layers as in most of NTK-related results.
\end{abstract}

\section{Introduction}
Understanding why gradient descent methods can often find a global optimum in minimizing the non-convex loss surface of neural nets 
is one of the key problems in deep learning theory.
Recently, this problem has been approached from the perspective of the Neural Tangent Kernel (NTK) \cite{JacotEtc2018}.
Consider an $L$-layer network with layer widths $\Set{n_0,\ldots,n_{L}}.$
Let $X\in\RR^{N\times n_0}$ and $Y\in\RR^{N\times n_L}$ be the training data, then the output at layer $l$ is given by
\begin{align}\label{eq:def_feature_map}
    F_l=\begin{cases}
	    X & l=0,\\
	    \sigma(F_{l-1}W_l) & l\in[L-1],\\
	    F_{L-1} W_L & l=L, 
        \end{cases}
\end{align}
where $W_l\in\RR^{n_{l-1}\times n_l}$, and $\sigma(x)=\max(0,x)$ applied componentwise.
The training loss is defined as
\begin{align*}
    \Phi(\theta) = \frac{1}{2}\norm{F_L(\theta)-Y}_F^2,
\end{align*}
where $\theta=(W_l)_{l=1}^{L}.$
The NTK matrix is
$K(\theta)=\left[\frac{\partial\vec(F_L)}{\partial\vec(\theta)}\right] \left[\frac{\partial\vec(F_L)}{\partial\vec(\theta)}\right]^T.$
Our starting observation is that
\begin{align*}
    \norm{\nabla\Phi(\theta)}_2^2 \geq 2\evmin{K(\theta)} \Phi(\theta).
\end{align*}
This resembles the Polyak-Lojasiewicz (PL) inequality \cite{Polyak1963} except that we do not have a constant on the RHS.
Nevertheless, it suggests that if $\evmin{K(\theta)}$ is bounded away from zero, both at initialization and during training,
then GD may still converge to a global optimum.
To show this property at initialization, one can resort to standard concentration tools, which we discuss later.
To show it holds throughout training, one popular way in the literature is to bound 
the traveling distance of GD from the initialization and the local Lipschitz constant of $K(\theta)$, and then invoke Weyl's eigenvalue inequality.
As $K$ involves the derivatives of $\sigma$ at various activation neurons,
it turns out that bounding the local Lipschitz constant of the NTK matrix can be done conveniently for sufficiently smooth activations \cite{DuEtal2019}.
However, when $\sigma'$ is not Lipschitz continuous (e.g.\ for ReLU), the entire analysis may become much more involved.
In particular, most of the existing proofs for ReLU nets require to study various quantities related to 
the changes of the activation patterns during training, see e.g.\ \cite{AllenZhuEtal2018, ZouGu2019}.
In this paper, we discuss a simple alternative proof, where such analysis is not needed.
In fact, starting from the following decomposition of the empirical NTK matrix
\begin{align}
    K=\sum_{l=1}^{L} \left[\frac{\partial\vec(F_L)}{\partial\vec(W_l)}\right] \left[\frac{\partial\vec(F_L)}{\partial\vec(W_l)}\right]^T
\end{align}
one can separate the sum into two terms: the first term involves $\sum_{l=1}^{L-1}(\cdot)$, 
and the second term corresponds to the derivative at $W_L$, which is given by $\Id_{n_L} \otimes (F_{L-1}F_{L-1}^T).$
Since both terms are positive semidefinite, one obtains
\begin{align}\label{eq:key}
    \evmin{K}\geq\evmin{F_{L-1}F_{L-1}^T}.
\end{align}
From here, in order to obtain a PL-like inequality, 
it suffices to bound the RHS of \eqref{eq:key} at initialization and keep track of $F_{L-1}$ during training.
This can be done efficiently without studying the changes of activation patterns.
Let us highlight that this is different from most of the prior works, where the output layer is fixed
and only the hidden ones are optimized, in which case the PL analysis necessarily involves $\sigma'$ from the hidden layers.
Besides that, prior works also need to study other ``smoothness'' properties of the loss
(local Lipschitz property of the gradient, or that of the Jacobian of the network),
which again requires bounds on the various quantities related to the changes of the activation patterns at different layers,
and thus make their proofs more involved.

In this work, we present an alternative proof framework, which does not require any analysis of the activation pattern changes.
It consists of two key steps mentioned above:
(a) bounding $\evmin{F_{L-1}F_{L-1}^T}$ at initialization, 
and (b) bounding the changes of $F_{L-1}$ during training. 
To show the convergence, we introduce a small trick to relate the loss of each iteration with that of the previous one.
That means, an explicit analysis of the ``smoothness'' of the loss is not needed.
Lastly, we show that our proof leads to improved bounds on layer widths in terms of the dependency on $N$.

\paragraph{Main Contributions.} 
First, we provide some easy-to-check sufficient conditions on the initialization
under which GD is guaranteed to converge to a global optimum.
Then, we show that all these conditions are satisfied 
when the width of the last hidden layer exceeds the number of training samples (all the remaining layers can have constant widths).
For LeCun's initialization,
we show that these conditions are satisfied (and consequently, GD finds a global optimum)
if the last hidden layer has quadratic width in case of two-layer networks, or cubic width in case of deep architectures.
Although results with similar orders of over-parameterization have been recently obtained
for deep nets with {\em smooth} activation functions \cite{HuangYau2020,QuynhMarco2020},
this is the first time, to the best of our knowledge, that such a result is proved for deep ReLU models.

\section{Main Result}\label{sec:main}
The GD updates are given by: $\theta_{k+1} = \theta_k - \eta \nabla\Phi(\theta_k)$.
Define the shorthand $F_l^k=F_l(\theta_k).$
We omit the argument $\theta$ and write just $F_l$ when it is clear from the context.

Let us first state some useful inequalities (see Lemma B.2 of \cite{QuynhMarco2020}).
There, the proof mainly uses the triangle inequality, and the Lip. property of ReLU.
\begin{lemma}\label{lem:merged}
    For every $\theta=(W_p)_{p=1}^L$ and $l\in[L],$ it holds
    \begin{align}
	\norm{\nabla_{W_l}\Phi}_F 
	&\leq\norm{X}_F\prod_{\substack{p=1\\p\neq l}}^L\norm{W_p}_2 \norm{F_L-Y}_F.\label{eq:grad_norm}
    \end{align}      
   Furthermore, let $\theta_a=(W_l^a)_{l=1}^L, \theta_b=(W_l^b)_{l=1}^L$, 
   and  $\max(\norm{W_l^a}_2, \norm{W_l^b}_2)\leq \bar{\lambda}_l$ for some $\bar{\lambda}_l\in\RR.$ 
   Then, for every $l\in[L]$ we have
    \begin{align}
	&\norm{F_l(\theta_a)-F_l(\theta_b)}_F\nonumber\\
	&\leq \norm{X}_F \left(\prod_{l=1}^{l} \bar{\lambda}_l\right) \sum_{p=1}^l \bar{\lambda}_p^{-1} \norm{W_p^a-W_p^b}_2 \label{eq:Lip_Fl}
    \end{align}
\end{lemma}

Our main theorem is the following.
\begin{theorem}\label{thm:general}
    Consider a deep ReLU network \eqref{eq:def_feature_map} where the width of the last hidden layer satisfies $n_{L-1}\geq N.$
    Let $(C_l)_{l=1}^{L}$ be any sequence of positive numbers.
    Define the following quantities:
    \begin{equation}\label{eq:notation_init}
	\alpha_0=\svmin{F_{L-1}^0},\;
	\bar{\lambda}_l = \norm{W_l^0}_2 + C_l,\;
	\bar{\lambda}_{i\to j} = \prod_{l=i}^j\bar{\lambda}_l .
    \end{equation}
    Assume that the following conditions are satisfied at the initialization:
    \begin{align}
	&\alpha_0^2 \geq 16\norm{X}_F \max_{l\in[L]} \frac{\bar{\lambda}_{1\to L}}{\bar{\lambda}_l C_l} \sqrt{2\Phi(\theta_0)} \label{eq:initW}\\
	&\alpha_0^3 \geq 32\norm{X}_F^2 \bar{\lambda}_{L} \sum_{l=1}^{L-1} \frac{\bar{\lambda}_{1\to L-1}^2}{\bar{\lambda}_l^{2}} \sqrt{2\Phi(\theta_0)} \label{eq:initF} \\
	&\alpha_0^2 \geq 16\norm{X}_F^2 \bar{\lambda}_L^2 \sum_{l=1}^{L-1} \frac{\bar{\lambda}_{1\to L-1}^2}{\bar{\lambda}_l^{2}} \label{eq:initS} 
    \end{align}
    Let the learning rate satisfy
    \begin{align}\label{eq:eta}
	\eta < \min\left( \frac{8}{\alpha_0^2},\; \norm{X}_F^{-2} \bar{\lambda}_{1\to L}^{-2} \left[\sum_{l=1}^{L-1}\bar{\lambda}_l^{-2}\right] \left[\sum_{l=1}^L\bar{\lambda}_l^{-2}\right]^{-2} \right)
    \end{align}
    Then the loss converges to a global minimum as 
    \begin{align}
	\Phi(\theta_k)\leq \left(1-\eta\frac{\alpha_0^2}{8}\right)^k\ \Phi(\theta_0),
    \end{align}
    for every $k\geq 0.$
\end{theorem}
\begin{proof}
    We show by induction for every $k\geq 0$ that
    \begin{align}\label{eq:induction_assump}
	\begin{cases}
	    \norm{W_l^r}_2\leq \bar{\lambda}_l, \quad l\in[L], \,r\in[0,k],\\
	    \svmin{F_{L-1}^r} \geq \frac{1}{2}\alpha_0,\quad r\in[0,k],\\
	    \Phi(\theta_r)\leq \Big(1-\eta\frac{\alpha_0^2}{8}\Big)^r \Phi(\theta_0),\quad r\in[0,k],
	\end{cases} 
    \end{align}
    Clearly, \eqref{eq:induction_assump} holds for $k=0.$
    Assume that \eqref{eq:induction_assump} holds up to iteration $k$.
    By the triangle inequality,
    \begin{align*}
	&\norm{W_l^{k+1}-W_l^0}_F 
	\leq\sum_{s=0}^{k} \norm{W_l^{s+1}-W_l^s}_F \\
	&=\eta\sum_{s=0}^{k} \norm{\nabla_{W_l}\Phi(\theta_s)}_F \\
	&\leq\eta\sum_{s=0}^{k} \norm{X}_F \prod_{\substack{p=1\\p\neq l}}^L\norm{W_p^s}_2 \norm{F_L^s-Y}_F\\
	    &\leq\eta\norm{X}_F \bar{\lambda}_{1\to L} \bar{\lambda}_l^{-1}
	    \sum_{s=0}^k \Big(1-\eta\frac{\alpha_0^2}{8}\Big)^{s/2} \norm{F_L^0-Y}_F, 
    \end{align*}	
    where the 2nd inequality follows from \eqref{eq:grad_norm}, and the last one follows from induction assumption. 
    Let $u\bydef\sqrt{1-\eta\alpha_0^2/8}$. The RHS of the previous expression is bounded as
    \begin{align*}
	&\frac{8}{\alpha_0^2} \norm{X}_F \bar{\lambda}_{1\to L} \bar{\lambda}_l^{-1} (1-u^2) \frac{1-u^{k+1}}{1-u} \norm{F_L^0-Y}_F\\
	&\leq \frac{16}{\alpha_0^2} \norm{X}_F \bar{\lambda}_{1\to L} \bar{\lambda}_l^{-1} \norm{F_L^0-Y}_F,\quad \textrm{ since } u\in (0, 1)\\
	&\leq C_l, \quad\textrm{by \eqref{eq:initW}} .
    \end{align*}
    By Weyl's inequality, this implies
    \begin{align}\label{eq:W_op}
	\norm{W_l^{k+1}}_2\leq \norm{W_l^0}_2 + C_l = \bar{\lambda}_l
    \end{align}
    Next, we have
    \begin{align*}
	&\norm{F_{L-1}^{k+1}-F_{L-1}^0}_F\\
	&\leq\norm{X}_F \bar{\lambda}_{1\to L-1} \sum_{l=1}^{L-1} \bar{\lambda}_l^{-1} \norm{W_l^{k+1}-W_l^0}_2,\quad \textrm{by \eqref{eq:Lip_Fl}}\\
	&\leq\frac{16}{\alpha_0^2}\norm{X}_F^2 \bar{\lambda}_{1\to L-1} \bar{\lambda}_{1\to L} \sum_{l=1}^{L-1} \bar{\lambda}_l^{-2} \sqrt{2\Phi(\theta_0)} \\
	&\leq\frac{1}{2} \alpha_0, \quad \textrm{by \eqref{eq:initF}}.
    \end{align*}
    This implies that $\svmin{F_{L-1}^{k+1}}\geq \frac{1}{2}\alpha_0.$
    So far, we have proved that the first two inequalities in \eqref{eq:induction_assump} hold for $k+1.$
    It remains to show that the third one also holds at $k+1.$
    Let us define the matrix $G=F_{L-1}^k W_L^{k+1}.$  
    Then, one has
    \begin{align*}
	&2\Phi(\theta_{k+1})\\
	&= 2\Phi(\theta_k) + \norm{F_L^{k+1}-F_L^k}_F^2 + 2\tr(F_L^{k+1}-F_L^k)(F_L^k-Y)^T\\
	&= 2\Phi(\theta_k) + \norm{F_L^{k+1}-F_L^k}_F^2 + 2\tr(F_L^{k+1}-G)(F_L^k-Y)^T\\
	&\quad + 2\tr(G-F_L^k)(F_L^k-Y)^T .
    \end{align*}
    Let us bound each term individually. 
    Using \eqref{eq:grad_norm}-\eqref{eq:Lip_Fl}, we have
    \begin{align*}
	\norm{F_L^{k+1}-F_L^k}_F
	&\leq\norm{X}_F \bar{\lambda}_{1\to L} \sum_{l=1}^{L} \bar{\lambda}_l^{-1} \norm{W_l^{k+1}-W_l^k}_2\\
	&\leq \eta \underbrace{\norm{X}_F^2 \bar{\lambda}_{1\to L}^2 \sum_{l=1}^{L} \bar{\lambda}_l^{-2}}_{Q_1} \norm{F_L^k-Y}_F .
    \end{align*}
    Furthermore, we have $F_L^{k+1}-G=(F_{L-1}^{k+1}-F_{L-1}^{k})W_L^{k+1}$, and thus it holds
    \begin{align*}
	&\tr(F_L^{k+1}-G)(F_L^k-Y)^T\\
	&\leq \norm{F_{L-1}^{k+1}-F_{L-1}^k}_F \norm{W_L^{k+1}}_2 \norm{F_L^k-Y}_F \\
	&\leq \eta \underbrace{\norm{X}_F^2 \bar{\lambda}_{1\to L-1}^2 \bar{\lambda}_L^2 \sum_{l=1}^{L-1} \bar{\lambda}_l^{-2}}_{Q_2} \norm{F_L^k-Y}_F^2,\; \textrm{by \eqref{eq:Lip_Fl}, \eqref{eq:W_op}}
    \end{align*}
    Lastly, by using the definition of $G$ and the fact that 
    $\nabla_{W_L}\Phi(\theta_k)=(F_{L-1}^k)^T (F_L^k-Y)$, we get
    \begin{align*}
	&\tr(G-F_L^k)(F_L^k-Y)^T\\
	&= -\eta \tr((F_{L-1}^k)^T (F_L^k-Y)(F_L^k-Y)^T F_{L-1}^k)\\
	&\leq -\eta \frac{\alpha_0^2}{4} \norm{F_L^k-Y}_F^2 
    \end{align*}
    where we used our assumption $n_{L-1}\geq N$ to obtain $\evmin{(F_{L-1}^k)^T(F_{L-1}^k)}=\svmin{F_{L-1}^k}^2,$
    and the induction assumption implies that $\svmin{F_{L-1}^k}\geq \frac{1}{2}\alpha_0.$
    Define $Q_1,Q_2$ as above. Putting all these bounds together, we get
    \begin{align*}
	\Phi(\theta_{k+1}) 
	&\leq 
	\left[ 1-\eta\frac{\alpha_0^2}{4}
	+ \eta^2 Q_1^2 
	+ \eta Q_2 
	\right]
	\Phi(\theta_k) \\
	&\leq \left[ 1-\eta \left[\frac{\alpha_0^2}{4} - 2Q_2 \right] \right] \Phi(\theta_k) && \textrm{By \eqref{eq:eta}} \\ 
	&\leq \left[ 1-\eta \frac{\alpha_0^2}{8} \right] \Phi(\theta_k) && \textrm{By }\eqref{eq:initS}
    \end{align*}
\end{proof}
We remark that Theorem \ref{thm:general} also holds for networks with biases.
Furthermore, its statement is meaningful for $\alpha_0>0$,
in which case there exists a $\theta$ such that $F_{L-1}$ has full row rank,
and thus the network can fit arbitrary labels.

Convergence of gradient descent for { deep ReLU nets} under the { square loss} has been studied in \cite{AllenZhuEtal2018,ZouGu2019},
where the widths of all the hidden layers scales as order of $N^8\textrm{poly}(L)$.
Here, Theorem \ref{thm:general} shows that a single wide layer of width $N$ suffices, 
and in particular, all the remaining layers can have arbitrary constant widths.
We remark that these results are not directly comparable at this point since it is not clear 
how likely the conditions \eqref{eq:initW}-\eqref{eq:initS} are satisfied when a common/practical initialization is considered.
We address this question in the next section.
%

\section{Applications of Theorem \ref{thm:general}}\label{sec:apps}
This section discusses the satisfiability of the conditions \eqref{eq:initW}-\eqref{eq:initS} of Theorem \ref{thm:general}
under different initialization schemes. 
To do so, it suffices to show a lower bound on $\alpha_0$ (as given in \eqref{eq:notation_init}), 
and an upper bound on both $\bar{\lambda}_l$ and the initial loss $\sqrt{2\Phi(\theta_0)}$,
and finally plugging these bounds into the above sufficient conditions to get a concrete requirement on layer widths.
More details are given in the following sections.

\subsection{Satisfiability of \eqref{eq:initW}-\eqref{eq:initS}: Width $N$ Suffices}\label{sec:N}
First, let us show that all the conditions \eqref{eq:initW}-\eqref{eq:initS} of Theorem \ref{thm:general} are satisfied
for a subset of points in parameter space. 
In particular, this does not require any additional requirements on the layer widths. 
Thus, if GD starts from one of these initial points, 
then it will converge to a global optimum provided that the width of the last hidden layer is at least $N$.

We apply Theorem \ref{thm:general} for $C_l=1, l\in[L].$
Let $\theta_0=(W_1^0,\ldots,W_{L-1}^0,W_L^0=0),$
where $(W_l^0)_{l=1}^{L-1}$ are chosen s.t.\ $\alpha_0=\svmin{F_{L-1}(\theta_0)}>0.$
Since $W_L^0=0$, we have $\sqrt{2\Phi(\theta_0)}=\norm{Y}_F.$
Let $\tilde{\theta}_0=(\beta W_1^0,\ldots,\beta W_{L-1}^0,W_L=0).$
Then, the condition \eqref{eq:initW} is satisfied at $\tilde{\theta}_0$ if
\begin{align}\label{eq:cond1}
    \beta^{2(L-1)} \alpha_0^2 \geq 16 \norm{X}_F \norm{Y}_F \prod_{l=1}^{L-1} (\beta\norm{W_l^0}_2+1) .
\end{align}
The LHS of the above inequality is a polynomial of degree $2(L-1)$ in $\beta$, whereas the RHS is a polynomial of degree $L-1.$
Thus, once the parameters of $\theta_0$ are fixed (as above), the inequality \eqref{eq:cond1} will be satisfied for a sufficiently large value of $\beta.$
Similarly, the condition \eqref{eq:initS} is satisfied at $\tilde{\theta}_0$ if
\begin{align}\label{eq:cond2}
    \beta^{2(L-1)} \alpha_0^2 \geq 16 \norm{X}_F^2 \sum_{l=1}^{L-1} \prod_{\substack{p=1\\p\neq l}}^{L-1} (\beta\norm{W_p^0}_2+1)^2.
\end{align}
The LHS of \eqref{eq:cond2} has degree $2(L-1)$, whereas the RHS has degree $2(L-2).$
Thus, \eqref{eq:cond2} is satisfied for large enough $\beta.$
The same argument applies to \eqref{eq:initF}.
As a result, all the conditions of Theorem \ref{thm:general} are met at $\tilde{\theta}_0$ for large enough $\beta.$

\subsection{LeCun's Initialization for Two-Layer ReLU Nets: Width $N^2$ Suffices}\label{sec:LeCun}
In this section, we show that all the conditions \eqref{eq:initW}-\eqref{eq:initS} of Theorem \ref{thm:general} are satisfied
for two-layer networks under LeCun's initialization,
provided that a stronger condition on the width is satisfied, namely $n_{L-1}=\Omega(N^2).$
The more general case of deep architecture is analyzed in Section \ref{sec:LeCun_deep}.

For simplicity, we assume the following assumptions on the data (standard in the literature):
\emph{(i)} the samples are i.i.d.\ sub-gaussian vectors with $\norm{X_{i:}}_2=\sqrt{d}$ and $\norm{X_{i:}}_{\psi_2}=\mathcal{O}(1)$,
\emph{(ii)} the ground-truth labels satisfy $\norm{Y_{i:}}_2=\mathcal{O}(1)$ for all $i\in[N]$,
and \emph{(iii)} the number of outputs $n_2$ is a constant.
For 2-layer nets, the conditions \eqref{eq:initW}-\eqref{eq:initS} become:
    \begin{align}
	&\alpha_0^2 \geq 16\norm{X}_F \max\left(\frac{\bar{\lambda}_2}{C_1},\frac{\bar{\lambda}_1}{C_2}\right) \sqrt{2\Phi(\theta_0)} \label{eq:initW2}\\
	&\alpha_0^3 \geq 32\norm{X}_F^2 \bar{\lambda}_{2} \sqrt{2\Phi(\theta_0)} \label{eq:initF2} \\
	&\alpha_0^2 \geq 16\norm{X}_F^2 \bar{\lambda}_2^2 \label{eq:initS2} 
    \end{align}
where $\alpha_0=\svmin{XW_1^0}$, $\bar{\lambda}_1=\norm{W_1^0}_2+C_1$, $\bar{\lambda}_2=\norm{W_2^0}_2+C_2$, $W_1^0\in\RR^{n_0\times n_1}$ and $W_2^0\in\RR^{n_1\times n_2}.$
Pick $C_1=C_2=1.$
Consider LeCun's initialization: 
\begin{align}
    (W_1^0)_{ij}\distas{}\mathcal{N}(0,1/n_0),\;
    (W_2^0)_{ij}\distas{}\mathcal{N}(0,1/n_1).
\end{align}
Then, by standard bounds on the operator norm of Gaussian matrices (see Theorem 2.12 of \cite{Davidson2001}), 
we have w.p.\ $\geq 1-e^{-\Omega(n_1)}$ that
\begin{align}
    \bar{\lambda}_1=\mathcal{O}\left(\frac{\sqrt{n_1}}{\sqrt{n_0}}\right),\quad
    \bar{\lambda}_2=\mathcal{O}(1).
\end{align}
Using Matrix-Chernoff inequality, one can easily show that (see e.g.\ Lemma 5.2 of \cite{QuynhNTK2021}) w.p.\ $\geq 1-\delta$,
\begin{align}
    \alpha_0\geq\sqrt{n_1\lambda_*/4},
\end{align}
as long as it holds $n_1\geq\tilde{\Omega}(N/\lambda_*)$,  
where $\lambda_*=\evmin{\E_{w\distas{}\mathcal{N}(0,n_0^{-1}\Id_{n_0})} [\sigma(Xw)\sigma(Xw)^T]}$ 
and $\tilde{\Omega}$ hides logarithmic factors depending on $\delta$.
Similarly, by using a standard concentration argument, we have w.p.\ $\geq 1-\delta$
\begin{align}
    \sqrt{2\Phi(\theta_0)}\leq\tilde{\mathcal{O}}(\sqrt{N}) .
\end{align}
The proof of this can be found in Lemma C.1 of \cite{QuynhMarco2020}.
Plugging all these bounds into \eqref{eq:initW2}-\eqref{eq:initS2},
we obtain that all these conditions are satisfied for $n_1=\Omega(N^2\lambda_*^{-2}.)$
Lastly, our data assumptions stated above allow us to apply Theorem 3.3 of \cite{QuynhMarco2020},
which implies w.h.p.\ that $\lambda_*\geq\Omega(1)$ as long as $d\geq N^r$ for some fixed integer constant $r\geq 2.$
Thus, all the conditions of Theorem \ref{thm:general} are satisfied for $n_1=\Omega(N^2)$.

As a remark, quadratic bounds on the layer width for two-layer ReLU nets have been also claimed in \cite{SongYang2020} 
(albeit for a different initialization).
The main differences consist in the fact that (i) we train all layers of the network while \cite{SongYang2020} train only the hidden layer, 
and (ii) our proof works for {\em arbitrary} depth,
and (iii) it does not require to bound the number of sign flips of the activation neurons,
nor the Lipschitz constant of the gradient.
This last property makes our proof significant simpler.

\subsection{LeCun's Initialization for Deep ReLU Nets: One Layer of Width $N^3$ Suffices}\label{sec:LeCun_deep}
This section considers the same type of initialization as in Section \ref{sec:LeCun}, except two differences:
i) we treat the case of deep nets, 
and ii) the variance of the output weights at the initialization is chosen to be smaller. 
Some previous studies e.g.\ \cite{AllenZhuEtal2018, ZouGu2019, ChenCaoZouGu2019} 
consider a regime in which the depth $L$ may grow with $N$.
In this section, we regard $L$ as a constant and obtain improved bounds in terms of the dependency on $N$.
More specifically, the best existing over-parameterization condition for the current setting (namely, deep ReLU nets + square loss) is given in
\cite{ZouGu2019}, where every hidden layer is required to have $\Omega(N^8)$ neurons.
The following corollary shows that all the conditions \eqref{eq:initW}-\eqref{eq:initS} of Theorem \ref{thm:general} are satisfied
for $n_{L-1}=\Omega(N^3)$. 
In particular, all the remaining layers can have arbitrary {\em sublinear} widths.

Let us assume that the training samples are drawn i.i.d.\ from the following family of distributions, denoted as $P_X$.
\begin{assumptions}\label{ass:data_dist}
    $P_X$ satisfies the following conditions:
    \begin{enumerate}
	\item For every Lipschitz continuous function $f:\RR^{n_0}\to\RR$, we have
	$\PP(\abs{f(x)-\E f(x)}>t)\leq 2\bigexp{-ct^2 / \norm{f}_{\Lip}^2}$ for all $t>0$, where $c>0$ is some absolute constant.
	\item $\E\norm{x}_2=\Theta(\sqrt{n_0})$, $\E\norm{x}_2^2=\Theta(n_0)$ and $\E\norm{x-\E x}_2^2=\bigOmg{n_0}.$ 
    \end{enumerate}
\end{assumptions}
Many distributions of interest satisfy Assumption \ref{ass:data_dist}, e.g.\ the standard Gaussian distribution and
the uniform distribution on the sphere.
In general, the first condition could cover
all those distributions satisfying the log-Sobolev inequality with data-independent constant.
The second condition is just about the scaling of the data. 
Here, we assume the data has norm of order $\sqrt{n_0}$, 
but it can also have other scalings.

The result of this section is stated below.
\begin{corollary}\label{cor:LeCun_deep}
    Consider a set of i.i.d.\ samples $\Set{X_{i:}}_{i=1}^{N}$ from a data distribution 
    which satisfies Assumption \ref{ass:data_dist}.
    Let the true labels satisfy $\norm{Y_{i:}}_2=\mathcal{O}(1)$ for all $i\in[N].$
    Fix any even integer constant $r>0$, and $\delta>0.$
    Let the widths of the network satisfy the following conditions:
    \begin{align}
	&n_l=\Theta(m),\quad\forall\,l\in[L-2],\\
	&n_{L-1}=\bigTildeOmg{N^3\max(m,n_0)^3}
    \end{align}
    where $m$ is some variable which can depend on $N$,
    and $\tilde{\Omega}$ hides logarithmic terms in $N$ and $\delta^{-1}.$
    Let the weights of the network be initialized as
    \begin{align}
	&(W_l)_{ij}\distas{}\mathcal{N}\left(0,\frac{1}{n_{l-1}}\right), \forall\,l\in[L-1],\\
	&(W_L)_{ij}\distas{}\mathcal{N}\left(0,\frac{1}{n_{L-1}^{4/3}}\right).
    \end{align}
    Then, for a small enough learning rate, gradient descent converges to a global optimum w.p. at least
    \begin{align*}
	&1 - \delta - N^2 \bigexp{ - \bigOmg{ \frac{\min(m,n_0)}{N^{2/(r-0.1)} \log(m)^{L-3}} } } .
    \end{align*}
\end{corollary}
We remark that the probability of Corollary \ref{cor:LeCun_deep} can be made arbitrarily close to $1$ as long as $\min(m,n_0)$ is not super-polynomially smaller than $N.$
For instance, one can pick $m,n_0=\tilde{\Omega}(N^{1/p})$ for a sufficiently large constant $p>0$,
in which case the main width condition of Corollary \ref{cor:LeCun_deep} becomes $n_{L-1}=\Omega(N^{3+\epsilon})$, for some small $\epsilon\in(0,1).$


In order to prove Corollary \ref{cor:LeCun_deep}, we use the following bound on $\alpha_0$ from \cite{QuynhNTK2021}.
Below is a restatement of their result for the special case of our interest (the smallest singular value of the last hidden layer).
\begin{lemma}\label{lem:bound_svmin_Fk}
    Let Assumption \ref{ass:data_dist} hold.
    Let $(W_l)_{ij}\distas{}\mathcal{N}(0,\beta_l^2)$ for all $l\in[L-1],i\in[n_{l-1}],j\in[n_l].$
    Fix any even integer $r>0$, and $\delta>0$.
    Suppose that
    \begin{align}
	&n_{L-1} = \bigOmg{N\log(N) \log\Big(\frac{N}{\delta}\Big)},\\
	&\prod_{l=1}^{L-3} \log(n_l)=\littleO{\min_{l\in[0,L-2]} n_l}.
    \end{align}
    Then one has
    \begin{align}
	\svmin{F_{L-1}}^2 
	=\bigTheta{ n_0\prod_{l=1}^{L-1} n_l\beta_l^2 }
    \end{align}
    w.p.\ 
	$\geq 1 - \delta - N^2 \bigexp{ - \bigOmg{ \frac{\min_{l\in[0,L-2]}n_l}{N^{2/(r-0.1)} \prod_{l=1}^{L-3}\log(n_l)} } } .$
\end{lemma}
We are now ready to prove the corollary.
\paragraph{Proof of Corollary \ref{cor:LeCun_deep}.}
    Let $a=\frac{4}{3},b=-\frac{1}{6},c=\frac{1}{2}.$
    By applying Theorem \ref{thm:general} for
    \begin{align*}
	C_l=
	\begin{cases}
	    1 & l\in[L-2],\\
	    n_{L-1}^{c} & l=L-1, \\
	    n_{L-1}^{b} & l= L,
	\end{cases}
    \end{align*}
    it suffices to show that all the conditions \eqref{eq:initW},\eqref{eq:initF},\eqref{eq:initS} are satisfied
    under the stated assumptions. 
    By Lemma \ref{lem:bound_svmin_Fk}, we have $\alpha_0^2 = \bigOmg{n_{L-1}}$ with the same probability as stated in the corollary.
    By Theorem 4.4.5 of \cite{vershynin2018high}, we have w.p.\ $\geq 1-e^{-\Omega(\min(n_0,m,n_{L-1}))}$ that
    \begin{align*}
	\bar{\lambda}_l 
	= 
	\begin{cases}
	    \bigO{n_{L-1}^{(1-a)/2} + n_{L-1}^b} & l=L\\
	    \bigO{\frac{\sqrt{n_{L-1}}}{\sqrt{m}} +  n_{L-1}^c} & l=L-1\\
	    \bigO{1} & l\in[2,L-2]\\
	    \bigO{\frac{\max(\sqrt{m},\sqrt{n_0})}{\sqrt{n_0}}} & l=1
	\end{cases}
    \end{align*}
    Define the shorthand $D=\norm{X}_F\sqrt{2\Phi(\theta_0)} .$
    Then, the condition \eqref{eq:initW} is satisfied for 
    \begin{align*}
	\begin{cases}
	    n_{L-1}^{1+b-\max(1/2,c)} \gtrsim D \max\left(1,\sqrt{\frac{m}{n_0}}\right),\\
	    n_{L-1}^{1+c-\max(b,(1-a)/2)} \gtrsim D \max\left(1,\sqrt{\frac{m}{n_0}}\right),\\
	    n_{L-1}^{1-\max(1/2,c)-\max(b,(1-a)/2)} \gtrsim D \max\left(1,\sqrt{\frac{m}{n_0}}\right) .
	\end{cases}
    \end{align*}
    Next, the condition \eqref{eq:initS} is satisfied for
    \begin{align*}
	n_{L-1}^{1-2\max(1/2,c)-2\max(b,(1-a)/2)} \gtrsim \max\left(1,\frac{m}{n_0}\right) \norm{X}_F^2 .
    \end{align*}
    Given the above condition, \eqref{eq:initF} is satisfied if it holds in addition that
    $\alpha_0\gtrsim \bar{\lambda}_L^{-1} \sqrt{2\Phi(\theta_0)}$, or equivalently,
    \begin{align*}
	n_{L-1}^{1/2+\max(b,(1-a)/2)} \gtrsim \sqrt{2\Phi(\theta_0)} .
    \end{align*}
    So far, all the conditions \eqref{eq:initW},\eqref{eq:initF},\eqref{eq:initS} are satisfied for
    \begin{align}\label{eq:star}
	\begin{cases}
	    n_{L-1}^{1/3} \gtrsim\norm{X}_F^2 \max\left(1,\frac{m}{n_0}\right),\\
	    n_{L-1}^{1/3} \gtrsim\norm{X}_F\sqrt{2\Phi(\theta_0)} \max\left(1,\frac{\sqrt{m}}{\sqrt{n_0}}\right),\\
	    n_{L-1}^{1/3} \gtrsim\sqrt{2\Phi(\theta_0)} .
	\end{cases}
    \end{align}
    By Lemma C.1 of \cite{QuynhMarco2020}, we have w.p.\ $\geq 1-e^{-\Omega(n_0)}$ that
    $\norm{F_L(\theta_0)}_F\lesssim\norm{X}_F.$
    Moreover, Assumption \ref{ass:data_dist} implies $\norm{x_i}=\Theta(n_0)$ w.p.\ $\geq 1-e^{-\Omega(n_0)}.$
    Taking the union bound, one obtains $\norm{X}_F=\mathcal{O}(\sqrt{Nn_0})$ w.p.\ $\geq 1-Ne^{-\Omega(n_0)}.$  
    By our data assumption, it holds $\norm{Y}_F=\mathcal{O}(\sqrt{N})$. 
    These last two facts imply $\sqrt{2\Phi(\theta_0)}\leq \norm{F_L(\theta_0)}_F+\norm{Y}_F\leq\mathcal{O}(\sqrt{Nn_0}).$
    Combining all the bounds above, one conclude that all the conditions in \eqref{eq:star} are satisfied for $n_{L-1}=\bigTildeOmg{N^3\max(m,n_0)^3}.$



\section{Further Related Work} 
Some earlier works consider linearly separable data, see e.g.\ \cite{SoudryEtal2018, BrutzkusEtal2018} and classification losses,
and they show that (stochastic) gradient descent can find a global (max margin) solution which generalizes well.
In contrast, the focus of this work is on the minimum level of over-parameterization required for general/worst-case data
(Theorem \ref{thm:general} makes no assumption on the training data) that GD finds a global optimum for a regression loss.

\cite{DuEtal2019} considers deep networks with the NTK parameterization \cite{JacotEtc2018} and show that
if all the hidden layers of the network have minimum width $\Omega(N^4)$, then GD converges to a global optimum.
Their proof requires to track the evolution of the various Gram matrices defined recursively through all the layers.
Recently, the width condition has been improved to $\Omega(N^3)$ by \cite{HuangYau2020} for the same setting.
\cite{QuynhMarco2020} shows the global convergence of GD for deep pyramidal nets with standard parameterization, 
where the width of the first hidden layer can scale linearly (or quadratically) in $N$.


For the setting of this paper, we are only aware of the results of \cite{AllenZhuEtal2018,ZouGu2019}.
One crucial difference which led to the proof of this work is that:
in our setup, all the layers are optimized with standard GD, 
whereas in the prior works, the input and output layers are fixed, and only the inner layers are optimized.
Interestingly, considering this slightly different setting allows us to reach a simple proof.
In particular, to obtain a PL-like inequality, they used the part of the NTK corresponding to the gradient of the last hidden layer,
whereas we used the gradient w.r.t\ the output layer.
A few other works also study convergence of GD for ReLU networks, albeit for two-layer models,
see e.g.\ \cite{arora2019fine, DuEtal2018_ICLR, OymakMahdi2019}.
Essentially, these results require that the number of hidden neurons scales at least as order of $N^4$.

Another closely related work is \cite{Amit2017}.
This paper shows that if the network is sufficiently large, 
the learning rate is small enough, the number of SGD steps is large enough,
then SGD can learn any function in the conjugate kernel space associated to the output of the last hidden layer.
Some crucial differences to our work are as follows.
First, they consider online SGD, whereas we consider standard GD on a fixed data set.
These two algorithms are different even if one sets the batch size to the size of the given training set, 
because in the former setting a fresh set of samples are drawn i.i.d.\ at every iteration, and thus the gradient 
at each iteration is not necessarily the same as the total gradient used by standard GD.
Second, their result requires the size of the network to scale inversely with $\textrm{poly}(\epsilon)$, 
where $\epsilon$ is the desired error. Thus, in order for the error to get down arbitrarily close to $0$, 
their network size needs to grow to infinity.
In contrast, our work provides guarantees for network sizes independent of $\epsilon$ (i.e.\ only dependent on $N$).
This difference also applies to several other works cited above.
Third, they do not provide guarantees that the loss is monotonically decreasing over the iterations like ours.
Fourth, their result shows that the loss after training is no worse than $\epsilon$ plus 
the minimum loss achievable by a function whose RKHS norm is bounded by a fixed variable $M$.
However, it is not clear what is this minimum loss for different choices of $M$.
This is especially important because their network size is also a polynomial function of $M$.
In contrast, our work provides an end-to-end result where the loss is ensured to converge w.h.p.\ to 0 for a fixed size network.
In terms of proof techniques, their proof is based on online SGD learning, 
while our proof is simpler and more transparent:
it is based on the fact that if the smallest singular value of the last hidden layer 
is bounded away from zero during training, then GD linearly converges to a global optimum.
Note that their proof as said is different, and it does not exploit this smallest singular value.
Finally, let us highlight that this single idea is not enough to prove the global convergence of the loss.
To show the convergence, most of the prior works require to study the local Lipschitz property of the Jacobian (or gradient of the loss),
which involves the study of various quantities related to the changes of activation patterns,
see e.g.\ \cite{AllenZhuEtal2018, ZouGu2019, OymakMahdi2019, SongYang2020, DuEtal2018_ICLR, arora2019fine}.
This has been a challenging task for deep ReLU nets since the derivative of ReLU is not Lipschitz continuous.
In contrast, the proof of Theorem \ref{thm:general} does not require such analysis.
The way it overcomes this issue is to introduce the transition matrix $G=F_{L-1}^k W_L^{k+1}$,
and then decompose the loss at each iteration according to $G$. 
As shown, the remaining steps of the proof follow mostly from the triangle inequality.
This makes the proof significantly simpler and more transparent.
Furthermore, we have shown that this proof leads to improved over-parameterization condition in terms of the dependency on the number of training samples.

\section*{Acknowledgement}
Quynh Nguyen acknowledges support from the European Research Council (ERC) under the European Union’s Horizon 2020 research 
and innovation programme (grant agreement no 757983).

\bibliography{regul}
\bibliographystyle{icml2021}

\end{document}